\documentclass[11pt,a4paper]{article}
\usepackage{amssymb,amsmath,latexsym}

\usepackage{fullpage}
\usepackage{amsfonts}
\usepackage{graphicx}
\usepackage[authoryear]{natbib}
\usepackage[font=small,format=plain,labelfont=bf,up]{caption}
\renewcommand{\citename}{\citet}
\renewcommand{\cite}{\citep}
\graphicspath{{./figures/}}

\usepackage{amsthm}
\usepackage{sg-macros}

\usepackage{authblk}
\usepackage{url}

\begin{document}
\title{On Differentiating Parameterized Argmin and Argmax Problems
with Application to Bi-level Optimization}
\author[1,2]{Stephen~Gould}
\author[1,3]{Basura~Fernando}
\author[1,3]{Anoop~Cherian}
\author[1,2]{Peter~Anderson}
\author[1,2]{Rodrigo~Santa~Cruz}
\author[1,2]{Edison~Guo}
\affil[1]{Australian Centre for Robotic Vision}
\affil[2]{ Research School of Computer Science, ANU}
\affil[3]{ Research School of Engineering, ANU \authorcr
  {\small \texttt{<firstname.lastname>@anu.edu.au}}
}
\maketitle

\begin{abstract}
Some recent works in machine learning and computer vision involve the
solution of a bi-level optimization problem. Here the solution of a
parameterized lower-level problem binds variables that appear in the
objective of an upper-level problem. The lower-level problem typically
appears as an $\argmin$ or $\argmax$ optimization problem. Many
techniques have been proposed to solve bi-level optimization problems,
including gradient descent, which is popular with current end-to-end
learning approaches. In this technical report we collect some results
on differentiating $\argmin$ and $\argmax$ optimization problems with
and without constraints and provide some insightful motivating
examples.
\end{abstract}

\section{Introduction}

Bi-level optimization has a long history of study dating back to the
early 1950s and investigation of the so-called Stackelberg
model~\cite{Bard:1998}. In this model, two players---a market leader
and a market follower---compete for profit determined by the price
that the market is willing to pay as a function of total goods
produced and each player's production cost. Recently, bi-level
optimization problems have found application in machine learning and
computer vision where they have been applied to parameter and
hyper-parameter learning~\cite{Do:NIPS07, Domke:AISTATS12,
  Klatzer:CVWW2015}, image denoising~\cite{Samuel:CVPR09, Ochs:2015},
and most recently, video activity recognition~\cite{Fernando:ICML2016,
  Fernando:CVPR2016}.

A bi-level optimization problem consists of an upper problem and a
lower problem. The former defines an objective over two sets of
variables, say $\bx$ and $\by$. The latter binds $\by$ as a function
of $\bx$, typically by solving a minimization problem. Formally, we
can write the problem as
\begin{align}
  \begin{array}{ll}
    \text{minimize}_{\bx} & f^{U}(\bx, \by) 
    \\
    \text{subject to} & \by \in \argmin_{\by'} f^{L}(\bx, \by')
  \end{array}
  \label{eqn:bilevel}
\end{align}
where $f^U$ and $f^L$ are the upper- and lower-level objectives,
respectively. As can be seen from the structure of the problem, the
lower-level (follower) optimizes its objective subject to the value of
the upper-level variable $\bx$. The goal of the upper-level (leader)
is to choose $\bx$ (according to its own objective) knowing that the
lower-level will follow optimally.

In one sense the $\argmin$ appearing in the lower-level problem is
just a mathematical function and so bi-level optimization can be
simply viewed as a special case of constrained optimization. In
another sense it is useful to consider the structure of the problem
and study bi-level optimization in its own right, especially when the
$\argmin$ cannot be computed in closed-form. As such, many techniques
have been proposed for solving different kinds of bi-level
optimization problems~\cite{Bard:1998, Dempe:2015}. In this technical
report we focus on first-order gradient based techniques, which have
become very important in machine learning and computer vision with the
wide spread adoption of deep neural network
models~\cite{Lecun:Nature15, Schmidhuber:NN15, Krizhevsky:NIPS12}.

Our main aim is to collect results on differentiating parameterized
$\argmin$ and $\argmax$ problems. Some of these results have appeared
in one form or another in earlier works, \eg\citename[\S
  5.6.2--5.6.3]{Faugeras:1993} considers the case of unconstrained and
equality constrained $\argmin$ problems when analysing uncertainty in
recovering 3D geometry. However, with the growth in popularity of deep
learning we feel it important to revisit the results (and present
examples) in the context of first-order gradient procedures for
solving bi-level optimization problems.

We begin with a brief overview of methods for solving bi-level
optimization problems to motivate our results in
\secref{sec:background}. We then consider unconstrained variants of
the lower-level problem, either $\argmin$ or $\argmax$ (in
\secref{sec:unconstrained}), and then extend the results to problems
with equality and inequality constraints (in
\secref{sec:constrained}). We include motivating examples with
gradient calculations and discussion throughout the paper leading to a
small bi-level optimization example for learning a novel softmax
classifier in \secref{sec:bilevel_example}. Examples are accompanied
by supplementary Python code.\footnote{Available for download from
  \url{http://users.cecs.anu.edu.au/sgould/}.}

\section{Background}
\label{sec:background}

The canonical form of a bi-level optimization problem is shown in
\eqnref{eqn:bilevel}. There are three general approaches to solving
such problems that have been proposed over the years. In the first
approach an analytic solution is found for the lower-level problem,
that is, an explicit function $\by^\star(\bx)$ that when evaluated
returns an element of $\argmin_{\by} f^{L}(\bx, \by)$. If such a
function can be found then we are in luck because we now simply solve
the single-level problem
\begin{align}
  \begin{array}{ll}
    \text{minimize}_{\bx} & f^{U}(\bx, \by^\star(\bx)) 
  \end{array}
\end{align}
Of course this problem, may itself be difficult to solve. Moreover, it
is not always the case that an analytic solution for the lower-level
problem will exist.

The second general approach to solving bi-level optimization problems
is to replace the lower-level problem with a set of sufficient
conditions for optimiality (\eg the KKT conditions for a convex
lower-level problem). If we think of these conditions being
encapsulated by the function $h^{L}(\bx, \by) = 0$ then we can solve the
following constrained problem instead of the original,
\begin{align}
  \begin{array}{ll}
    \text{minimize}_{\bx, \by} & f^{U}(\bx, \by) \\
    \text{subject to} & h^{L}(\bx, \by) = 0
  \end{array}
\end{align}
The main difficulty here is that the sufficient conditions may be hard
to express and the resulting problem hard to solve. Indeed, even if
the lower-level problem is convex, the resulting constrained problem
may not be.

The third general approach to solving bi-level optimization problems
is via gradient descent on the upper-level objective. The key idea is
to compute the gradient of the solution to the lower-level problem
with respect to the variables in the upper-level problem and perform
updates of the form
\begin{align}
  x &\leftarrow x - \eta \left. \left(\frac{\partial f^U}{\partial x}
  + \frac{\partial f^U}{\partial y} \frac{\partial y}{\partial
    x}\right) \right|_{(x, y^\star)}
\end{align}
In this respect the approach appears similar to the first
approach. However, now the function $\by^\star(\bx)$ does not need to
be found explicitly. All that we require is that the lower-level
problem be efficiently solveable and that a method exists for finding
the gradient at the current solution.\footnote{Note that we have made
  no assumption about the uniqueness of $y^\star$ nor the convexity of
  $f^L$ or $f^U$. When multiple minima of $f^L$ exist care needs to be
  taken during iterative gradient updates to select consistent
  solutions or when jumping between modes. However, these
  considerations are application dependent and do not invalidate any
  of the results included in this report.}

This last approach is important in the context large-scale and
end-to-end machine learning applications where first-order
(stochastic) gradient methods are often the preferred method. This
then motivates the results included in this technical report, \ie
computing the gradients of parameterized $\argmin$ and $\argmax$
optimization problems where the parameters are to be optimized for
some external objective or are themselves the output of some other
parameterized function to be learned.

\section{Unconstrained Optimization Problems}
\label{sec:unconstrained}

We begin by considering the gradient of the scalar function $g(x) =
\argmin_{y \in \reals} f(x, y)$ and follow an approach similar to the
implicit differentiation derivation that appears in earlier works
(\eg\cite{Samuel:CVPR09, Faugeras:1993}). In all of our results we
assume that the minimum (or maximum) over $y$ of the function $f(x,
y)$ exists over the domain of $x$. When the minimum (maximum) is not
unique then $g(x)$ can be taken as any one of the minimum (maximum)
points. Moreover, we do not need for $g(x)$ to have a closed-form
representation.

\begin{lemma}
Let $f: \reals \times \reals \rightarrow \reals$ be a continuous
function with first and second derivatives. Let $g(x) = \argmin_{y}
f(x, y)$. Then the derivative of $g$ with respect to $x$ is
\begin{align*}
  \frac{dg(x)}{dx} &= - \frac{f_{XY}(x, g(x))}{f_{YY}(x, g(x))}
\end{align*}
where $f_{XY} \doteq \frac{\partial^2f}{\partial x \partial y}$ and
$f_{YY} \doteq \frac{\partial^2f}{\partial y^2}$.
\label{lem:scalar_argmin}
\end{lemma}

\begin{proof}
\begin{align}
  \left. \frac{\partial f(x, y)}{\partial y} \right|_{y = g(x)} 
  &= 0
  & \text{(since $g(x) = \argmin_y f(x, y)$)}
  \\
  \therefore \, \frac{d}{dx} \frac{\partial f(x, g(x))}{\partial y}
  &= 0
  & \text{(differentiating lhs and rhs)}
\end{align}

But, by the chain rule,
\begin{align}
\frac{d}{dx} \left(\frac{\partial f(x, g(x))}{\partial y}\right)
&= \frac{\partial^2 f(x, g(x))}{\partial x \partial y} + \frac{\partial^2 f(x, g(x))}{\partial y^2} \frac{dg(x)}{dx}
\end{align}

Equating to zero and rearranging gives the desired result
\begin{align}
\frac{dg(x)}{dx}
&= - \left(\frac{\partial^2 f(x, g(x))}{\partial y^2}\right)^{-1} \frac{\partial^2 f(x, g(x))}{\partial x \partial y}
\\
&= - \frac{f_{XY}(x, g(x))}{f_{YY}(x, g(x))}
\end{align}
\end{proof}

We now extend the above result to the case of optimizing over
vector-valued arguments.

\begin{lemma}
Let $f: \reals \times \reals^n \rightarrow \reals$ be a continuous
function with first and second derivatives. Let $\bg(x) = \argmin_{\by \in \reals^n}
f(x, \by)$. Then the vector derivative of $\bg$ with respect to $x$ is
\begin{align*}
  \bg'(x) &= - f_{YY}(x, \bg(x))^{-1} f_{XY}(x, \bg(x)).
\end{align*}
where $f_{YY} \doteq \nabla^{2}_{\by\by} f(x, \by) \in \reals^{n \times n}$ and 
$f_{XY} \doteq \frac{\partial}{\partial x} \nabla_{\by} f(x, \by) \in \reals^{n}$.
\label{lem:vector_argmin}
\end{lemma}

\begin{proof}
Similar to Lemma~\ref{lem:scalar_argmin}, we have:
\begin{align}
f_{Y}(x, \bg(x)) \doteq {\nabla_{Y} f(x, \by)}|_{\by=\bg(x)} &= 0 \\
\frac{d}{dx} f_{Y}(x, \bg(x)) &= 0 \\
\therefore f_{XY}(x, \bg(x)) + f_{YY}(x, \bg(x)) \bg'(x) &= 0 \\
\frac{d}{dx}\bg(x) = - {f_{YY}(x, \bg(x))}^{-1} f_{XY}(x, \bg(x))
\end{align}
\end{proof}

The above lemma assumes that $\bg(x)$ is parameterized by a single
scalar value $x \in \reals$. It is trivial to extend the result to
multiple parameters $\bx = (x_1, \ldots, x_m)$ by performing the
derivative calculation for each parameter separately. That is,
\begin{align}
  \nabla_{\bx} \bg(x_1, \ldots, x_m) &= - f_{YY}(\bx, \bg(\bx))^{-1}
  \left[ \begin{matrix}
      f_{X_{1}Y}(\bx, \bg(\bx))
      & \cdots &
      f_{X_{m}Y}(\bx, \bg(\bx))
    \end{matrix} \right]
\end{align}

Note that the main computational challenge of the inverting the $n
\times n$ matrix $f_{YY}$ (or decomposing it to facilitate solving
each system of $n$ linear equations) only needs to be done once and
can then be reused for the derivative with respect to each
parameter. Thus the overhead of computing gradients for multiple
parameters is small compared to the cost of computing the gradient for
just one parameter. Of course if $\bx$ (or $\bg(\bx)$) is changed (\eg
during bi-level optimization) then $f_{YY}$ will be different and its
inverse recalculated.

So far we have only considered minimization problems. However,
studying the proofs above we see that they do not require that
$\bg(x)$ be a local-minimum point; any stationary point will
suffice. Thus, the result extends to the case of $\argmax$ problems as
well.

\begin{lemma}
Let $f: \reals \times \reals^n \rightarrow \reals$ be a continuous
function with first and second derivatives. Let $\bg(x) = \argmax_{\by \in \reals^n}
f(x, \by)$. Then the vector derivative of $\bg$ with respect to $x$ is
\begin{align*}
  \bg'(x) &= - f_{YY}(x, \bg(x))^{-1} f_{XY}(x, \bg(x)).
\end{align*}
where $f_{YY} \doteq \nabla^{2}_{\by\by} f(x, \by) \in \reals^{n \times n}$ and 
$f_{XY} \doteq \frac{\partial}{\partial x} \nabla_{\by} f(x, \by) \in \reals^{n}$.
\label{lem:vector_argmax}
\end{lemma}

\begin{proof}
Follows from proof of Lemma~\ref{lem:vector_argmin}.
\end{proof}

\subsection{Example: Scalar Mean}

In this section we consider the simple example of finding the point
whose sum-of-squared distance to all points in the set
$\{h_i(x)\}_{i=1}^{m}$ is minimized. Writing out the problem as a
mathematical optimization problem we have $g(x) = \argmin_y f(x, y)$
where $f(x, y) = \sum_{i=1}^{m} (h_i(x) - y)^2$. Here a well-known
analytic solution exists, namely the mean $g(x) =
\frac{1}{m}\sum_{i=1}^{m} h_i(x)$.

Applying \lemref{lem:scalar_argmin} we have:
\begin{align}
f_{Y}(x, y) &=  -2 \sum_{i=1}^{m} (h_i(x) - y) \\
f_{XY}(x, y) &=  -2 \sum_{i=1}^{m} h'_i(x) \\
f_{YY}(x, y) &= 2m \\
\therefore\, g'(x) &= \frac{1}{m} \sum_{i=1}^{m} h'_i(x)
\end{align}
which agrees with the analytical solution (assuming the derivatives
$h'_i(x)$ exist).

\subsection{Example: Scalar Function with Three Local Minima}
\label{sec:local_minima_example}

The results above do not require that the function $f$ being optimized
have a single global optimal point. Indeed, as discussed above, the
results hold for any stationary point (local optima or inflection
point). In this example we present a function with up to three
stationary points and show that we can calculate the gradient with
respect to $x$ at each stationary point.\footnote{Technically the
  function that we present only has three stationary points when $x <
  -2\left(\frac{4}{3}\right)^{\frac{1}{3}}$ or $x > 0$. It has two
  stationary points at $x = -2\left(\frac{4}{3}\right)^{\frac{1}{3}}$
  and a unique stationary point elsewhere.} Consider the function
\begin{align}
  f(x, y) &= xy^4 + 2x^2y^3 - 12y^2
\end{align}
with the following partial first- and second-order derivatives
\begin{align}
  f_Y(x, y) &= 4xy^3 + 6x^2y^2 - 24 y
  \\
  f_{XY}(x, y) &= 4y^3 + 12xy^2
  \\
  f_{YY}(x, y) &= 12xy^2 + 12x^2y - 24
\end{align}

Then for any stationary point $g(x)$ of $f(x, y)$, where stationarity
is with respect to $y$ for fixed $x$, we have
\begin{align}
  g'(x) &= -\frac{g(x)^3 + 3xg(x)^2}{3xg(x)^2 + 3x^2g(x) - 6}
\end{align}

Here the gradient describes how each stationary point $g(x)$ moves
locally with an infintisimal change in $x$. If such a problem, with
multiple local minima, were to be used within a bi-level optimization
learning problem then care should be taken during each iteration to
select corresponding solution points at each iteration of the
algorithm.

The (three) stationary points occur when $\frac{\partial
  f(x,y)}{\partial y} = 0$, which (in this example) we can compute
analytically as
\begin{align}
  g(x) \in \argmin_y f(x, y) = \left\{ 0, \frac{-3x^2 \pm \sqrt{9x^4 + 96x}}{4x} \right\}
\end{align}
which we use when generating the plots below.

\figref{fig:scalar_example} shows a surface plot of $f(x, y)$ and a
slice through the surface at $x = 1$. Clearly visible in the plot are
the three stationary points. The figure also shows the three values
for $g(x)$ and their gradients $g'(x)$ for a range of $x$. Note that
one of the solutions (\ie$y = 0$) is independent of $x$.

\begin{figure}
  \begin{tabular}{cc}
    \begin{minipage}{0.31\textwidth}
      \centering
      \includegraphics[width=\textwidth]{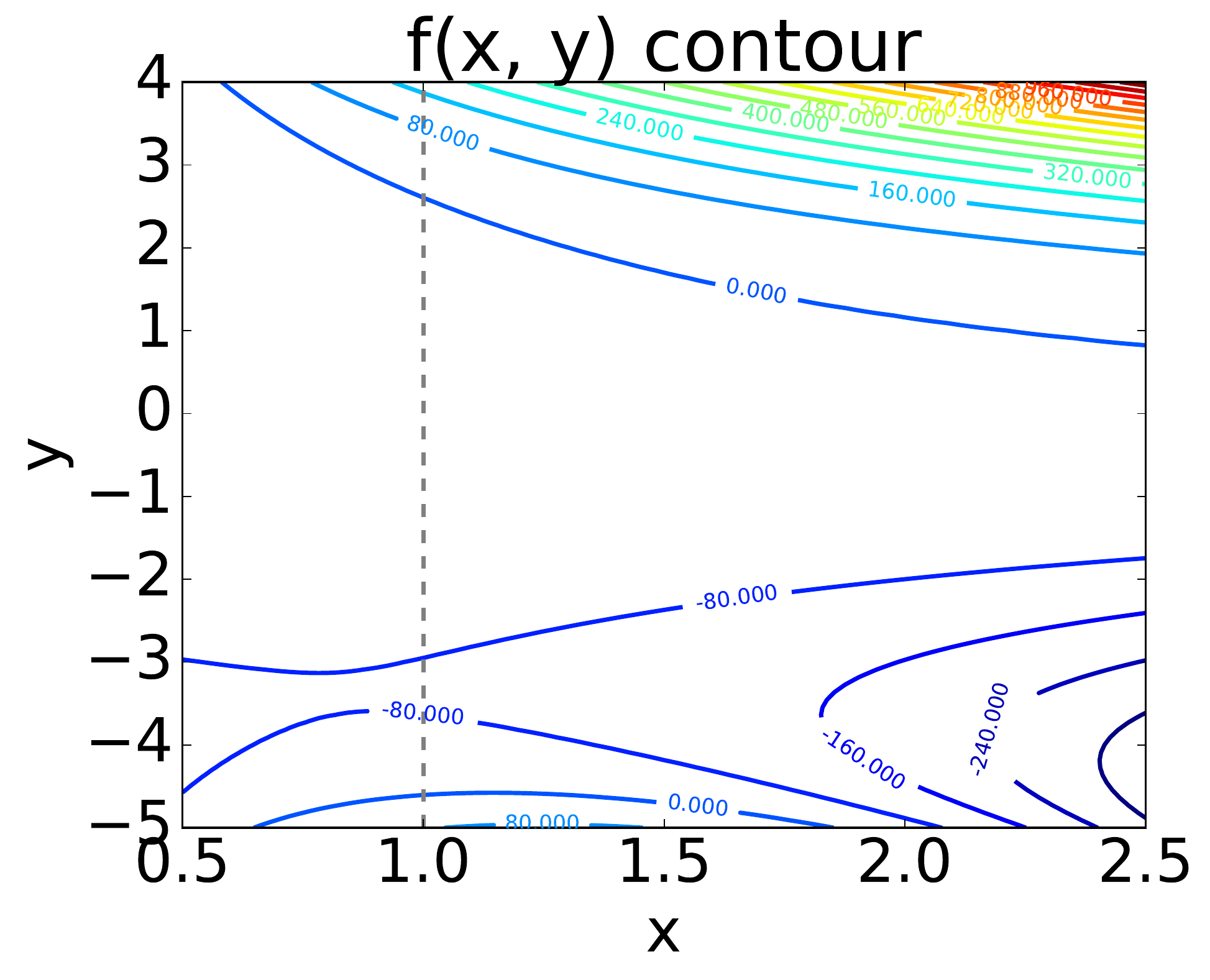} \\
      \includegraphics[width=\textwidth]{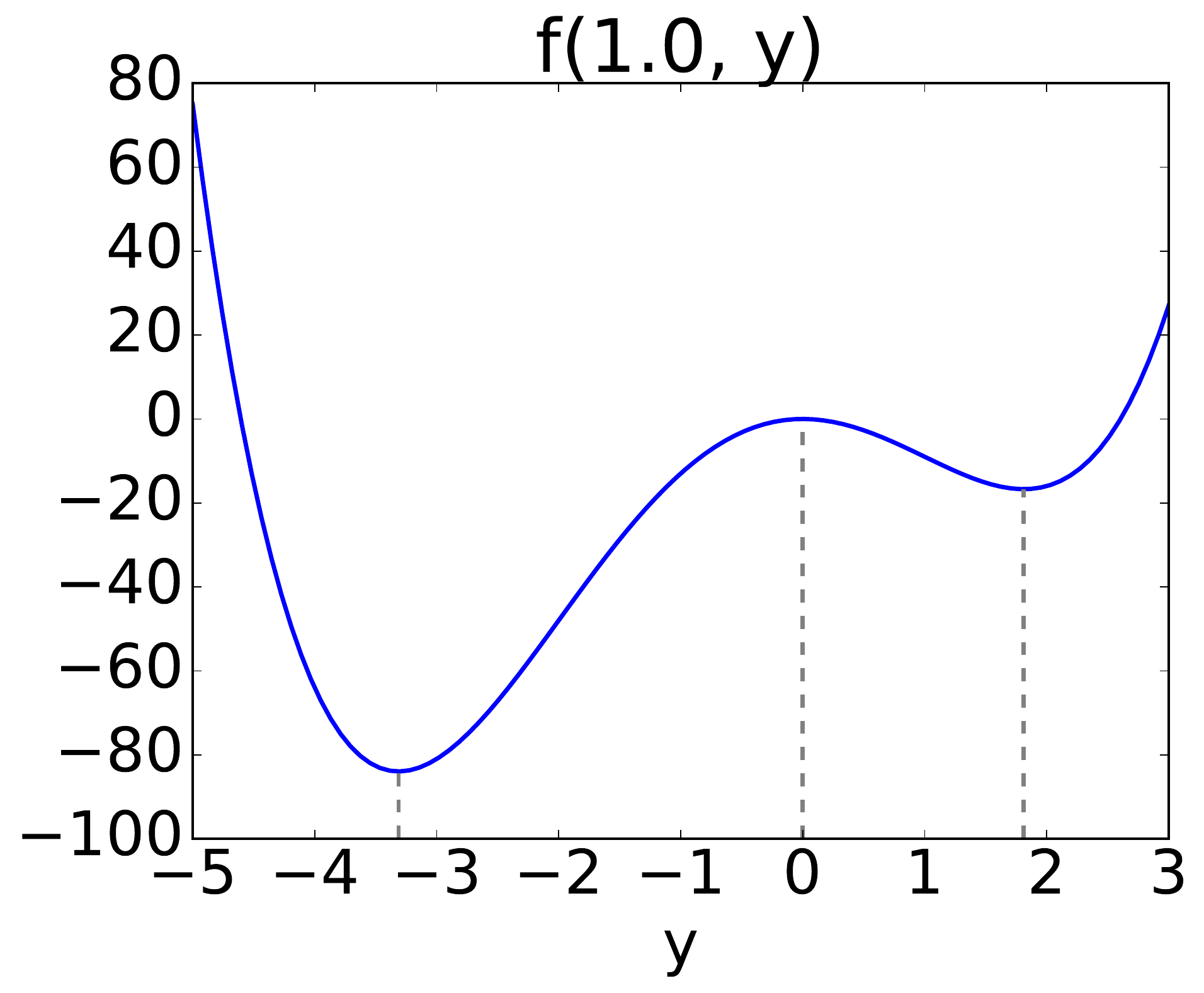}
    \end{minipage}
    &
    \begin{minipage}{0.64\textwidth}
      \centering
      \includegraphics[width=\textwidth]{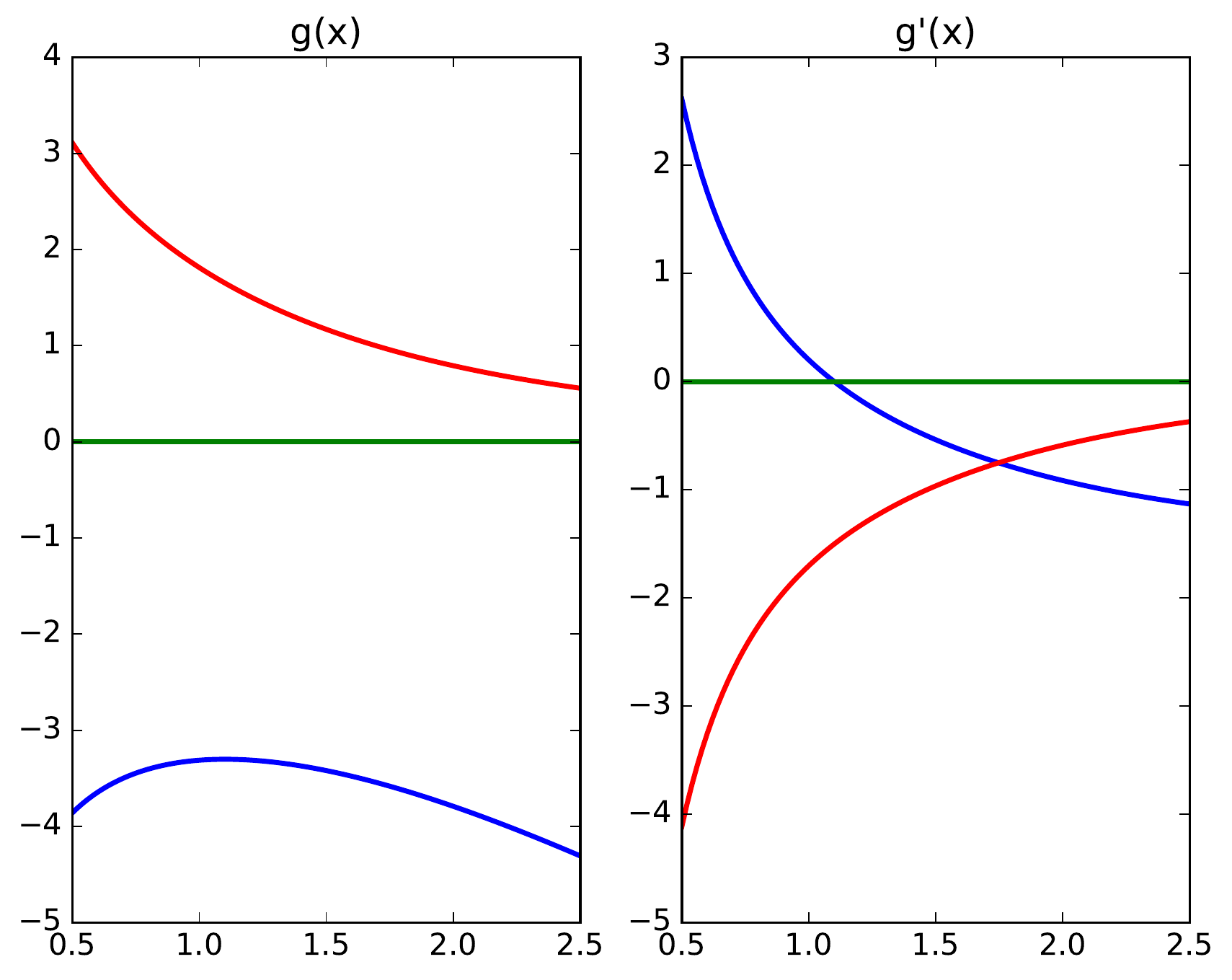}
    \end{minipage}
  \end{tabular}
  \caption{\label{fig:scalar_example} Example of a parameterized
    scalar function $f(x, y) = xy^4 + 2x^2y^3 - 12y^2$ with three
    stationary points for any fixed $x > 0$. The top-left panel shows
    a contour plot of $f$; the bottom-left panel shows the function at
    $x = 1$; and the remaining panels show the three solutions for
    $g(x) = \argmin_y f(x, y)$ and corresponding gradients $g'(x)$ at
    each stationary point.}
\end{figure}

\subsection{Example: Maximum Likelihood of Soft-max Classifier}
\label{sec:ml_example}

Now consider the more elaborate example of exploring how the maximum
likelihood feature vector of a soft-max classifier changes as a
function of the classifier's parameters. Assume $m$ classes and let
classifier be parameterized by $\Theta = \{(\ba_i,
b_i)\}_{i=1}^{m}$. Then we can define the likelihood of feature vector
$\bx$ for the $i$-th class of a soft-max distribution as
\begin{align}
  \ell_i(\bx) &= P(Y = i \mid \bX = \bx; \Theta) \\
  &= \frac{1}{Z(\bx; \Theta)} \exp\left(\ba_i^T \bx + b_i\right)
\end{align}
where $Z(\bx; \Theta) = \sum_{j=1}^{m} \exp\left(\ba_j^T \bx +
b_j\right)$ is the partition function.\footnote{Note that we use a
  different notation here to be consistent with the standard notation
  in machine learning. In particular, the role of $\bx$ is differs
  from its appearance elsewhere in this article.}

The maximum (log-)likelihood feature vector for class $i$ can be found as
\begin{align}
  \bg_i(\Theta) &= \argmax_{\bx \in \reals^n} \log \ell_i(\bx) \\
  &= \argmax_{\bx \in \reals^n} \left\{
  \ba_i^T \bx + b_i - \log \left( \sum_{j=1}^{m} \exp\left(\ba_j^T \bx + b_j\right) \right)
  \right\}
\end{align}
whose objective is concave and so has a unique global
maximum. However, the problem, in general, has no closed-form
solution. Yet we can still compute the derivative of the maximum
likelihood feature vector $\bg_i(\Theta)$ with respect to any of the
model parameters as follows.
\begin{align}
  \nabla_{\bx} \log \ell_i(\bx)
  &= \ba_i - \sum_{j=1}^{m} \ell_j(\bx) \ba_j
  \\
  \nabla^2_{\bx\bx} \log \ell_i(\bx)
  &= \sum_{j=1}^{m} \sum_{k=1}^{m} \ell_j(\bx) \ell_k(\bx) \ba_j \ba_k^T -
  \sum_{j=1}^{m} \ell_j(\bx) \ba_j\ba_j^T
  \\
  \frac{\partial}{\partial a_{jk}} \nabla_{\bx} \log \ell_i(\bx)
  &= \ind{i = j}\be_k - \ell_j(\bx) x_k \left(\ba_j - \sum_{l=1}^{m} \ell_l(\bx) \ba_l\right) - \ell_j(\bx) \be_k
  \\
  \frac{\partial}{\partial b_{j}} \nabla_{\bx} \log \ell_i(\bx)
  &= - \ell_j(\bx) \left(\ba_j - \sum_{k=1}^{m} \ell_k(\bx) \ba_k\right)
\end{align}
where $\be_k$ is the $k$-th canonical vector ($k$-th element one and
the rest zero) and $\ind{\cdot}$ is the indicator function (or iverson
bracket), which takes value 1 when its argument is true and 0
otherwise.

Letting $\bx^\star = \bg_i(\Theta)$, $H = \nabla^2_{\bx\bx} \log
\ell_i(\bx^\star)$ and $\bar{\ba} = \sum_{k=1}^{m} \ell_k(\bx^\star)
\ba_k$, we have
\begin{align}
  \frac{\partial \bg_i}{\partial a_{jk}}
  &= \begin{cases}
    H^{-1} (\ell_i(\bx^\star) - 1) \be_k,
    & i = j
    \\
    \ell_j(\bx^\star) H^{-1} \left(x^\star_k(\ba_j - \bar{\ba}) + \be_k\right),
    & i \neq j
  \end{cases}
  \\
  \frac{\partial \bg_i}{\partial b_{j}}
  &= \begin{cases}
    0, & i = j
    \\
    \ell_j(\bx^\star) H^{-1} \left(\ba_j - \bar{\ba}\right), & i \neq j    
  \end{cases}
\end{align}
where we have used the fact that $\ba_i = \bar{\ba}$ since
$\nabla_{\bx} \log \ell_i(\bx^\star) = 0$.

\begin{figure}
  \begin{tabular}{cc}
    \includegraphics[width=0.45\textwidth]{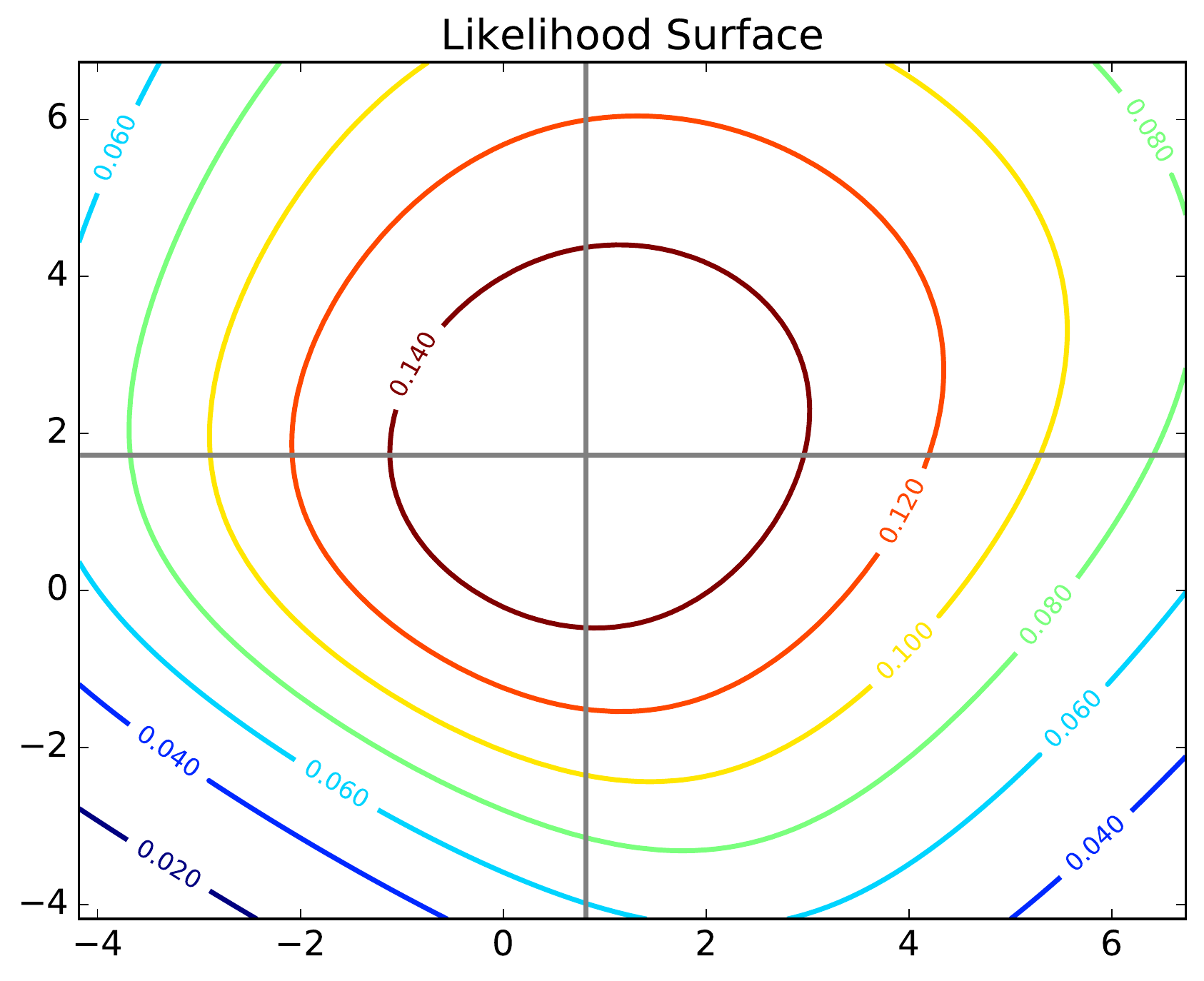} &
    \includegraphics[width=0.45\textwidth]{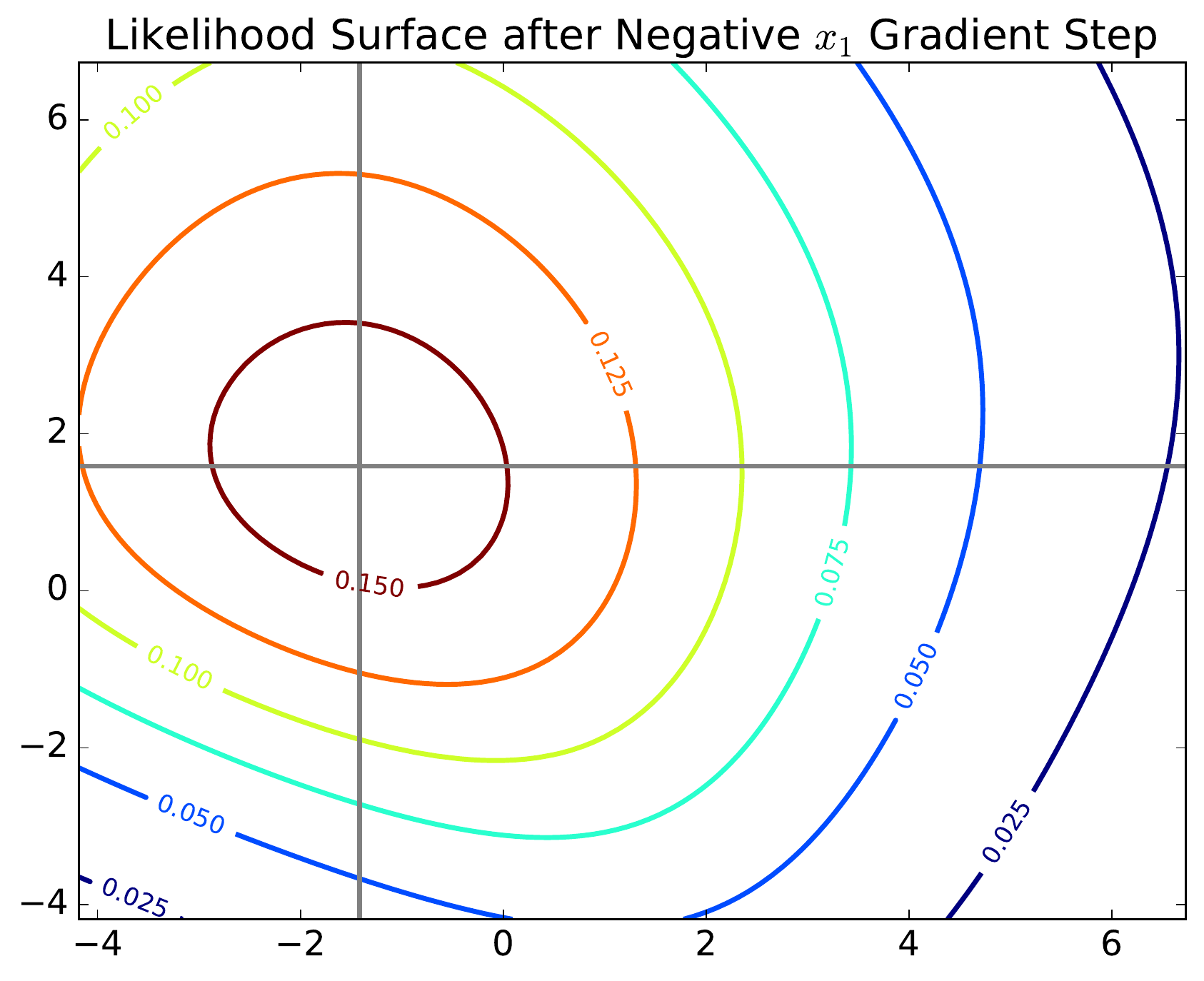}
    \\
    {\small (a) $\bx^\star = (0.817, 1.73)$} &
    {\small (b) $\bx^\star = (-1.42, 1.59)$}
  \end{tabular}
  \caption{\label{fig:ml_example} Example maximum-likelihood surfaces
    $\ell_i(\bx)$ before and after taking a small step on all
    parameters in the negative gradient direction for $x^\star_1$.}
\end{figure}

\figref{fig:ml_example} shows example maximum-likelihood surfaces for
one of the likelihood functions $\ell_i(\bx)$ from a soft-max
distribution over ten classes and two-dimensional feature vectors. The
parameters of the distribution were chosen randomly. Shown are the
initial surface and the surface after adjusting all parameters in the
negative gradient direction for the first feature dimension, \ie
$\theta \leftarrow \theta - \eta \be_1^T \frac{\partial
  \bg_i}{\partial \theta}$ where we have used $\theta$ to stand for
any of the parameters in $\{a_{ij}, b_i \mid i = 1, \ldots, 10; j = 1,
2\}$.

This example will be developed further below by adding equality and
inequality constraints and then applying it in the context of bi-level
optimization.

\subsection{Invariance Under Monotonic Transformations}
\label{sec:monotonic}

This section explores the idea that the optimal point of a function is
invariant under certain transformations, \eg exponentiation. Let $g(x)
= \argmin_{y} f(x, y)$. Now let $\tilde{f}(x, y) = e^{f(x, y)}$ and
$\tilde{g}(x) = \argmin_{y} \tilde{f}(x, y)$. Clearly, $\tilde{g}(x) =
g(x)$ (since the exponential function is smooth and monotonically
increasing).

Computing the gradients we see that
\begin{align}
\tilde{f}_Y(x, y) &= e^{f(x, y)} f_Y(x, y) \\
\tilde{f}_{XY}(x, y) &= e^{f(x, y)} f_{XY}(x, y) + e^{f(x, y)} f_X(x, y) f_Y(x, y) \\
\tilde{f}_{YY}(x, y) &= e^{f(x, y)} f_{YY}(x, y) + e^{f(x, y)} f_Y^2(x, y)
\end{align}

So applying Lemma~\ref{lem:scalar_argmin}, we have
\begin{align}
\tilde{g}'(x) &= - \frac{\tilde{f}_{XY}(x, g(x))}{\tilde{f}_{YY}(x, g(x))} \\
&= -\frac{e^{f(x, g(x))} f_{XY}(x, g(x)) + e^{f(x, g(x))} f_X(x, g(x)) f_Y(x, g(x))}
{e^{f(x, g(x))} f_{YY}(x, g(x)) + e^{f(x, g(x))} f_Y^2(x, g(x))} \\
&= -\frac{f_{XY}(x, g(x)) + f_X(x, g(x)) f_Y(x, g(x))}
{f_{YY}(x, g(x)) + f_Y^2(x, g(x))} & \text{(cancelling $e^{f(x, g(x))}$)} \\
&= -\frac{f_{XY}(x, g(x))}{f_{YY}(x, g(x))} & \text{(since $f_Y(x, g(x)) = 0$)} \\
&= g'(x)
\end{align}

Similarly, now assume $f(x, y) > 0$ for all $x$ and $y$, and let
$\tilde{f}(x, y) = \log f(x, y)$ and $\tilde{g}(x) = \argmin_{y}
\tilde{f}(x, y)$. Again we have $\tilde{g}(x) = g(x)$ (since the
logarithmic function is smooth and monotonically increasing on the
positive reals).

Verifying Lemma~\ref{lem:scalar_argmin}, we have
\begin{align}
\tilde{f}_Y(x, y) &= \frac{1}{f(x, y)} f_Y(x, y) \\
\tilde{f}_{XY}(x, y) &= \frac{f(x, y) f_{XY}(x, y) - f_Y(x, y)f_X(x, y)}{f^2(x, y)} \\
\tilde{f}_{YY}(x, y) &= \frac{f(x, y) f_{YY}(x, y) - f^2_Y(x, y)}{f^2(x, y)}
\end{align}
and once again we can set $f_Y(x, g(x))$ to zero and cancel terms to get $\tilde{g}'(x) = g'(x)$.

These examples motivate the following lemma, which formalizes the fact
that composing a function with a monotonically increasing or
monotonically decreasing function does not change its stationary
points.

\begin{lemma}
  Let $f: \reals \times \reals \rightarrow \reals$ be a continuous
  function with first and second derivatives. Let $h: \reals
  \rightarrow \reals$ be a smooth monotonically increasing or
  monotoncially decreasing function and let $g(x) = \argmin_{y} h(f(x,
  y))$. Then the derivative of $g$ with respect to $x$ is
  \begin{align*}
    \frac{dg(x)}{dx} &= - \frac{f_{XY}(x, g(x))}{f_{YY}(x, g(x))}
  \end{align*}
  where $f_{XY} \doteq \frac{\partial^2f}{\partial x \partial y}$
  and $f_{YY} \doteq \frac{\partial^2f}{\partial y^2}$.
\label{lem:scalar_argmin_composed}  
\end{lemma}

\begin{proof}
  Follows from \lemref{lem:scalar_argmin} observing that
  $\frac{\partial h(f(x, y))}{\partial y} = h'(f(x, y)) \frac{\partial
    f(x, y)}{\partial y}$ by the chain rule and that $h'(f(x, y))$ is
  always non-zero for monotonically increasing or decreasing $h$.
\end{proof}

\section{Constrained Optimization Problems}
\label{sec:constrained}

In this section we extend the results of the $\argmin$ and $\argmax$
derivatives to problems with linear equality and arbitrary inequality
constraints.

\subsection{Equality Constraints}
\label{sec:eq_constraints}

Let us introduce linear equality constraints $A \by = \bb$ into the
vector version of our minimization problem. We now have $\bg(x) =
\argmin_{\by : A \by = \bb} f(x, \by)$ and wish to find $\bg'(x)$.

\begin{lemma}
Let $f: \reals \times \reals^n \rightarrow \reals$ be a continuous
function with first and second derivatives. Let $A \in \reals^{m
  \times n}$ and $\bb \in \reals^m$. Let $\by_0 \in \reals^n$ be any
vector satisfying $A \by_0 = \bb$ and let the columns of $F$ span the
null-space of $A$. Let $\bz^\star(x) \in \argmin_{\bz} f(x, \by_0 + F
\bz)$ so that $\bg(x) = \by_0 + F \bz^\star(x)$. Then
\begin{align*}
  \bg'(x) &= - F \left(F^T f_{YY}(x, \bg(x)) F\right)^{-1} F^T f_{XY}(x, \bg(x))
\end{align*}
where $f_{YY} \doteq \nabla^{2}_{\by\by} f(x, \by) \in \reals^{n \times n}$ and 
$f_{XY} \doteq \frac{\partial}{\partial x} \nabla_{\by} f(x, \by) \in \reals^{n}$.
\end{lemma}

\begin{proof}
  \begin{align}
    \left. \nabla_{\bz} f(x, \by_0 + F \bz) \right|_{\bz = \bz^\star(x)} &= 0 \\
    F^T f_Y(x, \bg(x)) &= 0 \\
    F^T f_{XY}(x, \bg(x)) + F^T f_{YY}(x, \bg(x)) F \bz'(x) &= 0 \\
    \therefore \, \bz'(x) &= - (F^T f_{YY}(x, \bg(x)) F)^{-1} F^T f_{XY}(x, \bg(x)) \\
    \therefore \, \bg'(x) &= - F (F^T f_{YY}(x, \bg(x)) F)^{-1} F^T f_{XY}(x, \bg(x))
  \end{align}
\end{proof}

Alternatively, we can construct the $\bg'(x)$ directly as the
following lemma shows.

\begin{lemma}
Let $f: \reals \times \reals^n \rightarrow \reals$ be a continuous
function with first and second derivatives. Let $A \in \reals^{m
  \times n}$ and $\bb \in \reals^m$ with $\mathop{\textrm{rank}}(A) =
m$. Let $\bg(x) = \argmin_{\by : A\by = \bb} f(x, \by)$. Let $H =
f_{YY}(x, \bg(x))$. Then
\begin{align*}
  \bg'(x) &= \left(H^{-1} A^T \left(A H^{-1} A^T\right)^{-1} A H^{-1} - H^{-1} \right)
  f_{XY}(x, \bg(x))
\end{align*}
where $f_{YY} \doteq \nabla^{2}_{\by\by} f(x, \by) \in \reals^{n \times n}$ and 
$f_{XY} \doteq \frac{\partial}{\partial x} \nabla_{\by} f(x, \by) \in \reals^{n}$.
\label{lem:constrained_argmin}
\end{lemma}

\begin{proof}
  Consider the Lagrangian $\Ell$ for the constrained optimization problem,
  \begin{align}
    \begin{array}{ll}
      \displaystyle \mathop{\text{minimize}}_{\by \in \reals^n} & f(x, \by) \\
      \text{subject to} & A \by = \bb
    \end{array}
  \end{align}
  We have
  \begin{align}
    \Ell(x, \by, \blambda) &= f(x, \by) + \blambda^T (A \by - \bb)
  \end{align}
  Assume $\tilde{\bg}(x) = (\by^\star(x), \blambda^\star(x))$ is a
  optimal primal-dual pair. Then
  \begin{align}
    \left[ \begin{matrix}
        \nabla_{\by}\Ell(x, \by^\star, \blambda^\star) \\
        \nabla_{\blambda}\Ell(x, \by^\star, \blambda^\star)
      \end{matrix} \right] =
    \left[ \begin{matrix}
        f_Y(x, \by^\star) + A^T \blambda^\star \\
        A \by^\star - \bb
      \end{matrix} \right] &= 0
    \\
    \frac{\text{d}}{\text{d}x}
    \left[ \begin{matrix}
        f_Y(x, \by^\star) + A^T \blambda^\star \\
        A \by^\star - \bb
      \end{matrix} \right] &= 0
    \\
    \left[ \begin{matrix}
        f_{XY}(x, \by^\star) \\
        0
      \end{matrix} \right]
    +
    \left[ \begin{matrix}
        f_{YY}(x, \by^\star) & A^T \\
        A & 0
      \end{matrix} \right]
    \left[ \begin{matrix}
        \tilde{\bg}_Y(x) \\ \tilde{\bg}_{\Lambda}(x)
      \end{matrix} \right]
    &= 0
    \\
    \therefore \;
    \left[ \begin{matrix}
        f_{YY}(x, \by^\star) & A^T \\
        A & 0
      \end{matrix} \right]
    \left[ \begin{matrix}
        \tilde{\bg}_Y(x) \\ \tilde{\bg}_{\Lambda}(x)
      \end{matrix} \right]
    &=  \left[ \begin{matrix}
        - f_{XY}(x, \by^\star) \\
        0
      \end{matrix} \right]
  \end{align}

  From the first row of the block matrix equation above, we have
  \begin{align}
    \tilde{\bg}_Y(x) &= H^{-1} \left( - f_{XY}(x, \by^\star) - A^T \tilde{\bg}_\Lambda(x) \right)
  \end{align}
  where $H = f_{YY}(x, \by^\star)$.

  Substituting into the second row gives
  \begin{align}
    A H^{-1} \left(f_{XY}(x, \by^\star) + A^T \tilde{\bg}_\Lambda(x) \right) &= 0 \\
    \therefore \; \tilde{\bg}_\Lambda(x) &= - \left( A H^{-1} A^T \right)^{-1} A H^{-1} f_{XY}(x, \by^\star)
  \end{align}
  which when substituted back into the first row gives the desired
  result with $\bg(x) = \by^\star$ and $\bg'(x) = \tilde{\bg}_{Y}(x)$.
\end{proof}

Note that for $A = 0$ (and $b = 0$) this reduces to the result from
Lemma~\ref{lem:vector_argmin}. Furthermore, $\bg'(x)$ is in the
null-space of $A$, which we require if the linear equality constraint
is to remain satisfied.


\subsection{Inequality Constraints}
\label{sec:ineq_constraints}

Consider a family of optimization problems, indexed by $x$, with
inequality constraints. In standard form we have
\begin{align}
  \begin{array}{lll}
    \text{minimize}_{\by \in \reals} & f_0(x, \by) \\
    \text{subject to} & f_i(x, \by) \leq 0 & i = 1, \ldots, m
  \end{array}
\end{align}
where $f_0(x, \by)$ is the objective function and $f_i(x, \by)$ are
the inequality constraint functions.

Let $\bg(x) \in \reals^n$ be an optimal solution. We wish to find
$\bg'(x)$, which we approximate using ideas from interior-point
methods~\cite{Boyd:2004}. Introducing the log-barrier function
$\phi(x, \by) = \sum_{i=1}^{m} \log(-f_i(x, \by))$, we can approximate
the above problem as
\begin{align}
  \begin{array}{ll}
    \text{minimize}_{\by} & t f_0(x, \by) - \sum_{i=1}^{m} \log(-f_i(x, \by))
  \end{array}
\end{align}
where $t > 0$ is a scaling factor that controls the approximation (or
duality gap when the problem is convex). We now have an unconstrained
problem and can invoke Lemma~\ref{lem:vector_argmin}.

For completeness, recall that the gradient and Hessian of the
log-barrier function, $\phi(\bz) = \sum_{i=1}^{m} \log(-f_i(\bz))$,
are given by
\begin{align}
  \nabla \phi(\bz) & = - \sum_{i=1}^{m} \frac{1}{f_i(\bz)} \nabla f_i(\bz), \\
  \nabla^2 \phi(\bz) & =  \sum_{i=1}^{m} \frac{1}{f_i(\bz)^2} \nabla f_i(\bz) \nabla f_i(\bz)^T
  - \sum_{i=1}^{m} \frac{1}{f_i(\bz)} \nabla^2 f_i(\bz)
\end{align}
Thus the gradient of an inequality constrained $\argmin$ function can
be approximated as
\begin{align}
  \bg'(x) &\approx - \Big(t f_{YY}(x, \bg(x)) - \phi_{YY}(x, \bg(x))\Big)^{-1} 
  \Big(t f_{XY}(x, \bg(x)) - \phi_{XY}(x, \bg(x)) \Big)
\end{align}
In many cases the constraint functions $f_i$ will not depend on $x$
and the above expression can be simplified by setting $\phi_{XY}(x,
\by)$ to zero.

\subsection{Example: Positivity Constraints}

We consider an inequality constrained version of our scalar mean
example from above,
\begin{align}
  \begin{array}{lll}
    g(x) =& \argmin_{y \in \reals} & \sum_{i=1}^{m} \left(h_i(x) - y\right)^2 \\
    & \text{subject to} &  y \geq 0
  \end{array}
\end{align}
where we have added a positivity constraint on $y$. This problem has
closed-form solution
\begin{align}
  g(x) &= \max \left\{0,\, \frac{1}{m} \sum_{i=1}^{m} h_i(x)\right\}
\end{align}

Following \secref{sec:ineq_constraints} we can construct the
approximation
\begin{align}
  g_t(x) &= \argmin_{y \in \reals} t f(x, y) - \log(y)
\end{align}
where $f(x, y) = \sum_{i=1}^{m} \left(h_i(x) - y\right)^2$. Applying
\lemref{lem:scalar_argmin} gives
\begin{align}
  \frac{\partial}{\partial y} \left(t f(x, y) - \phi(y)\right)
  &= -2 t \sum_{i=1}^{m} \left(h_i(x) - y\right) - \frac{1}{y}
  \\
  \frac{\partial^2}{\partial x \partial y} \left(t f(x, y) - \phi(y)\right)
  &= -2t \sum_{i=1}^{m} h'_i(x)
  \\
  \frac{\partial^2}{\partial y^2} \left(t f(x, y) - \phi(y)\right)
  &= 2tm + \frac{1}{y^2}
  \\
  \therefore \; g'_t(x) &= \frac{2t \sum_{i=1}^{m} h'_i(x)}{2tm + \frac{1}{g(x)^2}}
  \label{eqn:pos_example_grad}
\end{align}

Note that here we can solve the quadratic equation induced by
$\frac{\partial}{\partial y} \left(t f(x, y) - \phi(y)\right) = 0$ to
obtain a closed-form solution of $g_t(x)$ and hence $g'_t(x)$. We
leave this as an exercise.

Observe from \eqnref{eqn:pos_example_grad} that as $t \rightarrow \infty$,
\begin{align}
  g'_t(x) &= \begin{cases}
    \frac{1}{m} \sum_{i=1}^{m} h'_i(x) & \text{if $g(x) > 0$} \\
    0 & \text{if $g(x) = 0$}
    \end{cases}
\end{align}

We demonstrate this example on a special case with $m = 1$ and $h_1(x)
= x$. The true and approximate function and their gradients are
plotted in \figref{fig:pos_example}.

\begin{figure}
  \begin{tabular}{cc}
    \includegraphics[width=0.45\textwidth]{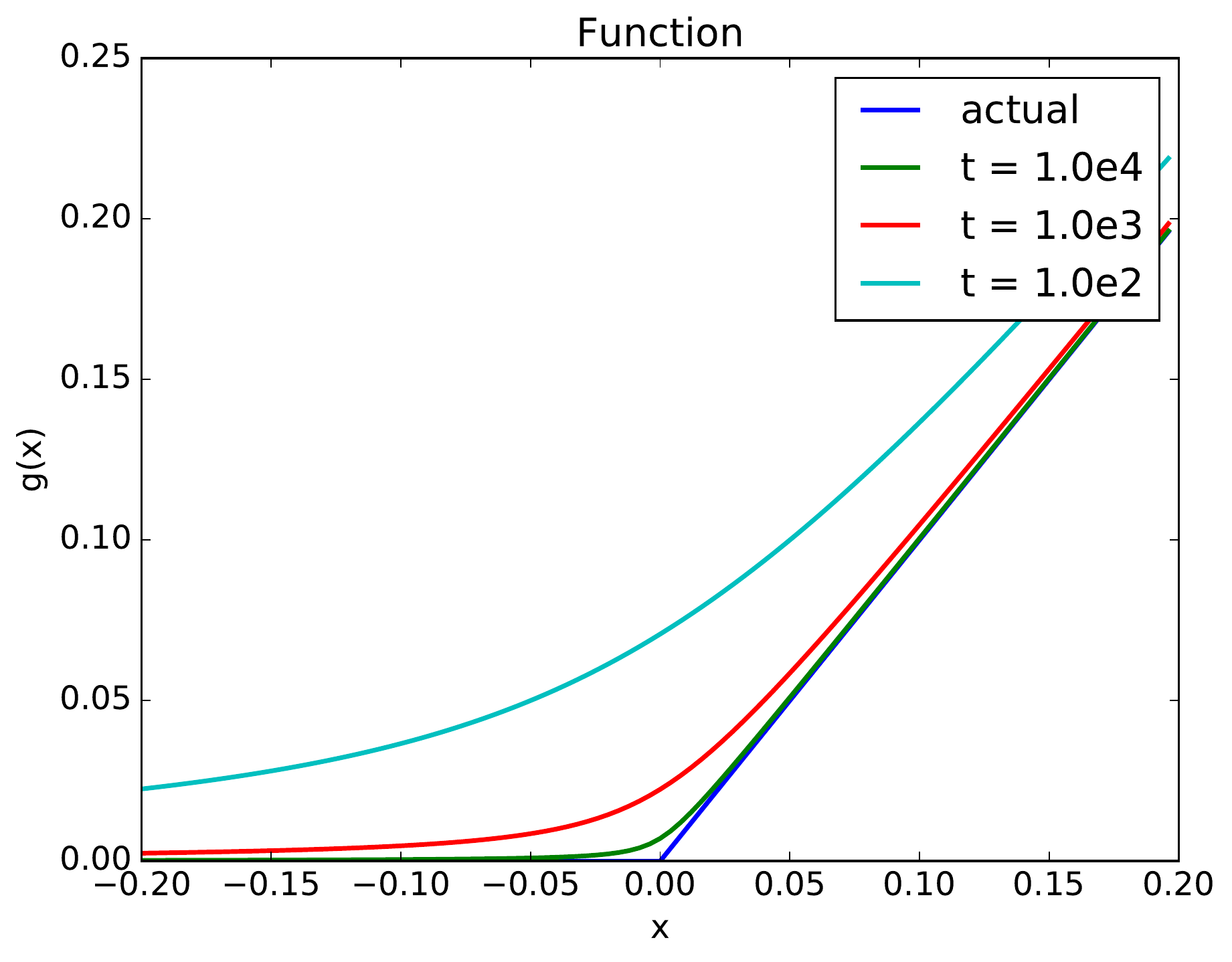} &
    \includegraphics[width=0.45\textwidth]{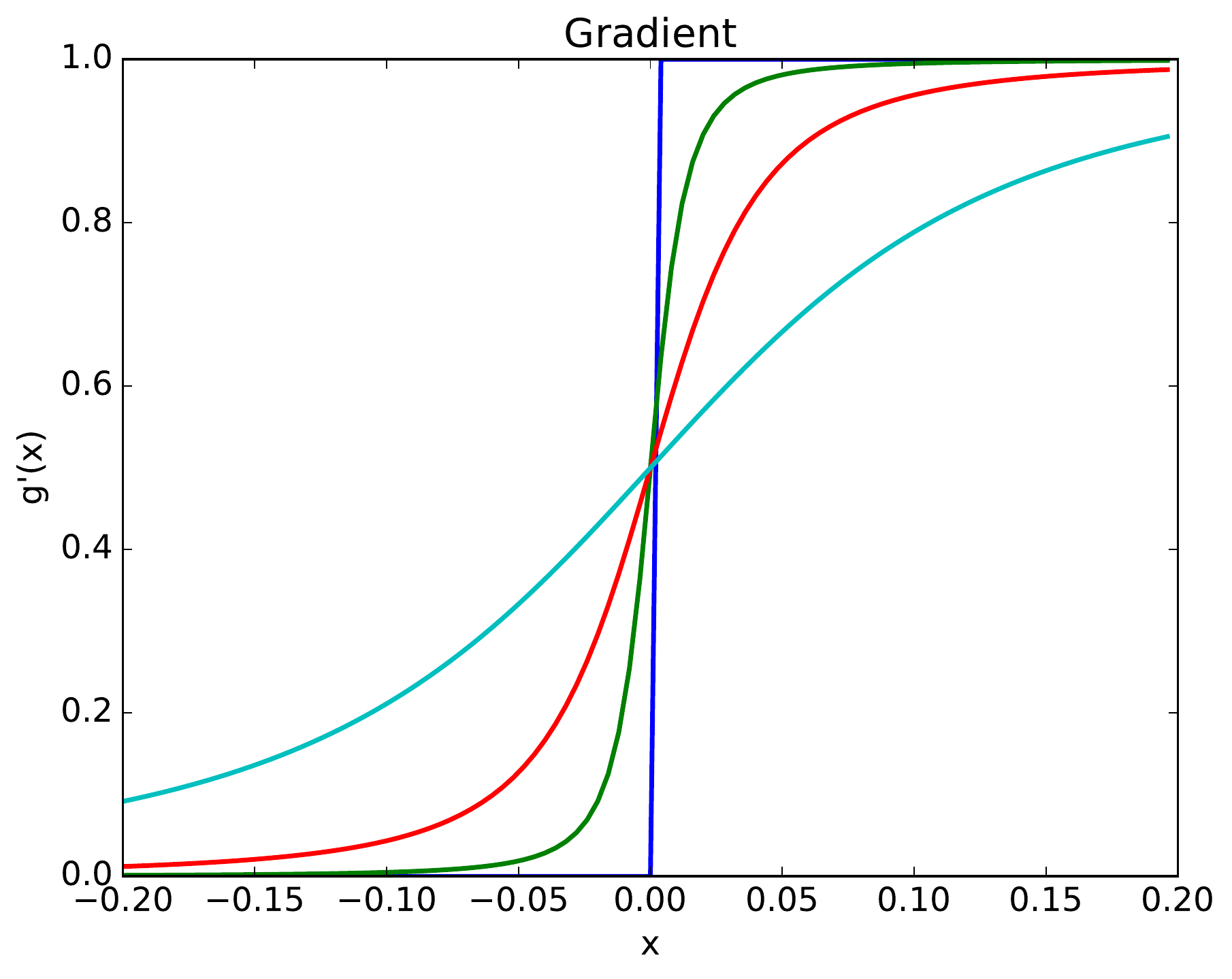}
  \end{tabular}
  \caption{\label{fig:pos_example} Example graph of both function
    value and gradient for $g(x) = \argmin_{y \geq 0} (x - y)^2$ and
    approximations $g_t(x) = \argmin_{y} t (x - y)^2 - \log(y)$ for
    different values of $t$. As $t \rightarrow \infty$ the
    approximation converges to the actual function and gradient.}
\end{figure}

\subsection{Example: Maximum Likelihood of Constrained Soft-max Classifier}
\label{sec:constrained_ml_example}

Let us now continue our soft-max example from \secref{sec:ml_example}
by adding constraints on the solution. We begin by adding the linear
equality constraint $\ones^T \bx = 1$ to give
\begin{align}
  \begin{array}{lll}
    \bg_i(\Theta)
    =& \argmax_{\bx \in \reals^n} & \left\{
    \ba_i^T \bx + b_i - \log \left( \sum_{j=1}^{m} \exp\left(\ba_j^T \bx + b_j\right) \right)
    \right\}
    \\
    & \text{subject to} & \ones^T \bx = 1
  \end{array}
\end{align}

Following \lemref{lem:constrained_argmin} and letting $H^\dagger =
\left(H^{-1} - \frac{1}{\ones^T H^{-1} \ones} H^{-1} \ones \ones^T
H^{-1}\right)$ we have the following gradients with respect to each
parameter,
\begin{align}
  \frac{\partial \bg_i}{\partial a_{jk}}
  &= \begin{cases}
    H^{\dagger} (\ell_i(\bx^\star) - 1) \be_k,
    & i = j
    \\
    \ell_j(\bx^\star) H^{\dagger} \left(x^\star_k(\ba_j - \bar{\ba}) + \be_k\right),
    & i \neq j
  \end{cases}
  \\
  \frac{\partial \bg_i}{\partial b_{j}}
  &= \begin{cases}
    0, & i = j
    \\
    \ell_j(\bx^\star) H^{\dagger} \left(\ba_j - \bar{\ba}\right), & i \neq j    
  \end{cases}
\end{align}
where $H$ and $\bar{\ba}$ are as defined in \secref{sec:ml_example}.

\figref{fig:eq_ml_example} shows the constrained solutions before and
after taking a gradient step for the same maximum-likelihood surface
as described in \secref{sec:ml_example}. Notice that the solution lies
along the line $x_1 + x_2 = 1$. Moreover, the sum of gradients for
$x_1$ and $x_2$ are zero, \ie the gradient is in the null-space of
$\ones^T$. To show this it is sufficient to prove that $H^\dagger
\ones = 0$, which is straightforward.

Next we consider adding the inequality constraint $\|\bx\|^2 \leq 1$
to the soft-max maximum-likelihood problem. That is, we constrain the
solution to lie within a unit ball centered at the origin. The new
function is given by
\begin{align}
  \bg_i(\Theta)
  =& \argmax_{\bx \in \reals^n} \left\{
  \ba_i^T \bx + b_i - \log \left( \sum_{j=1}^{m} \exp\left(\ba_j^T \bx + b_j\right) \right)
  \right\}
  \\
  & \text{subject to $\|\bx\|^2 \leq 1$}
\end{align}

\begin{figure}
  \begin{tabular}{cc}
    \includegraphics[width=0.45\textwidth]{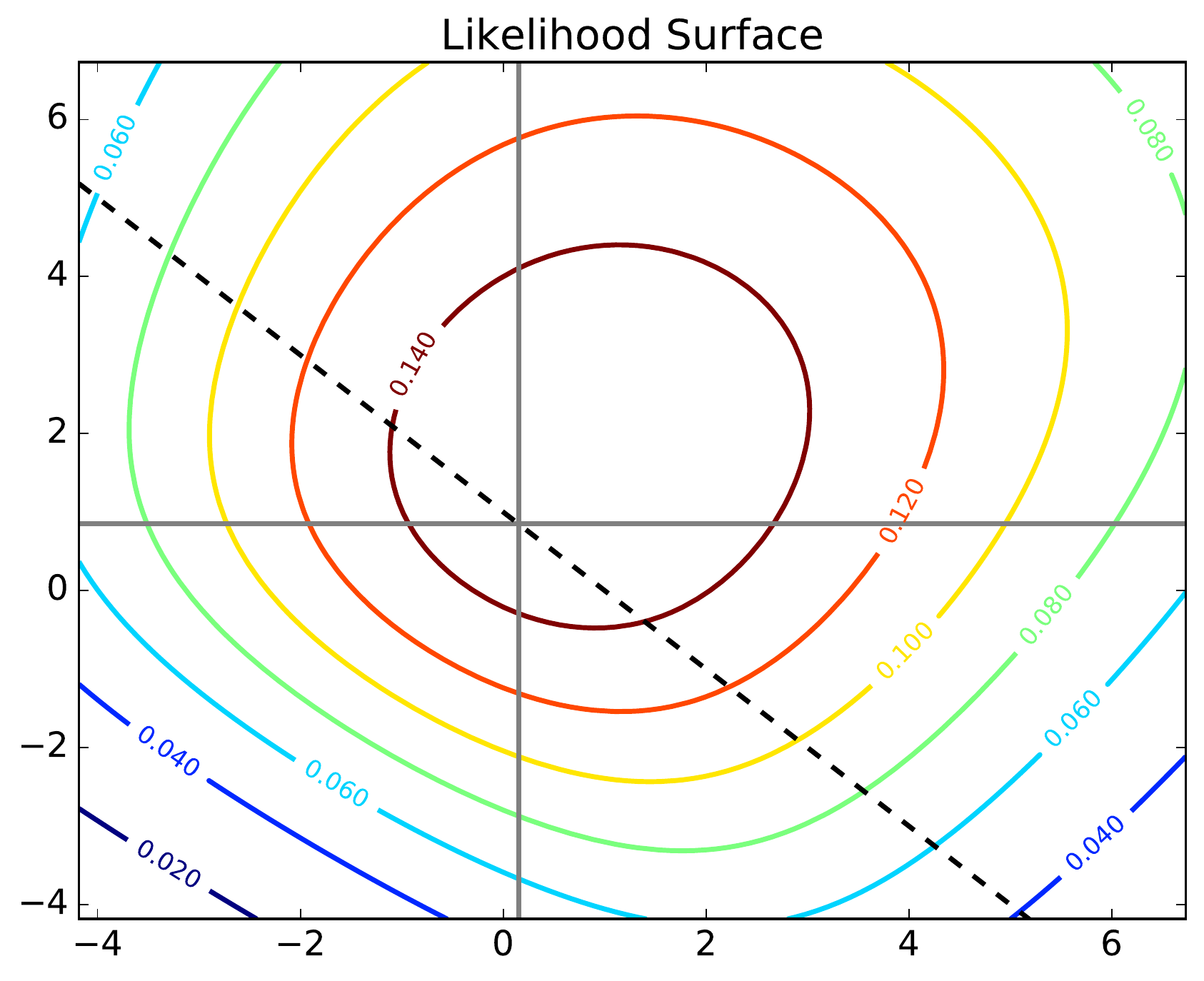} &
    \includegraphics[width=0.45\textwidth]{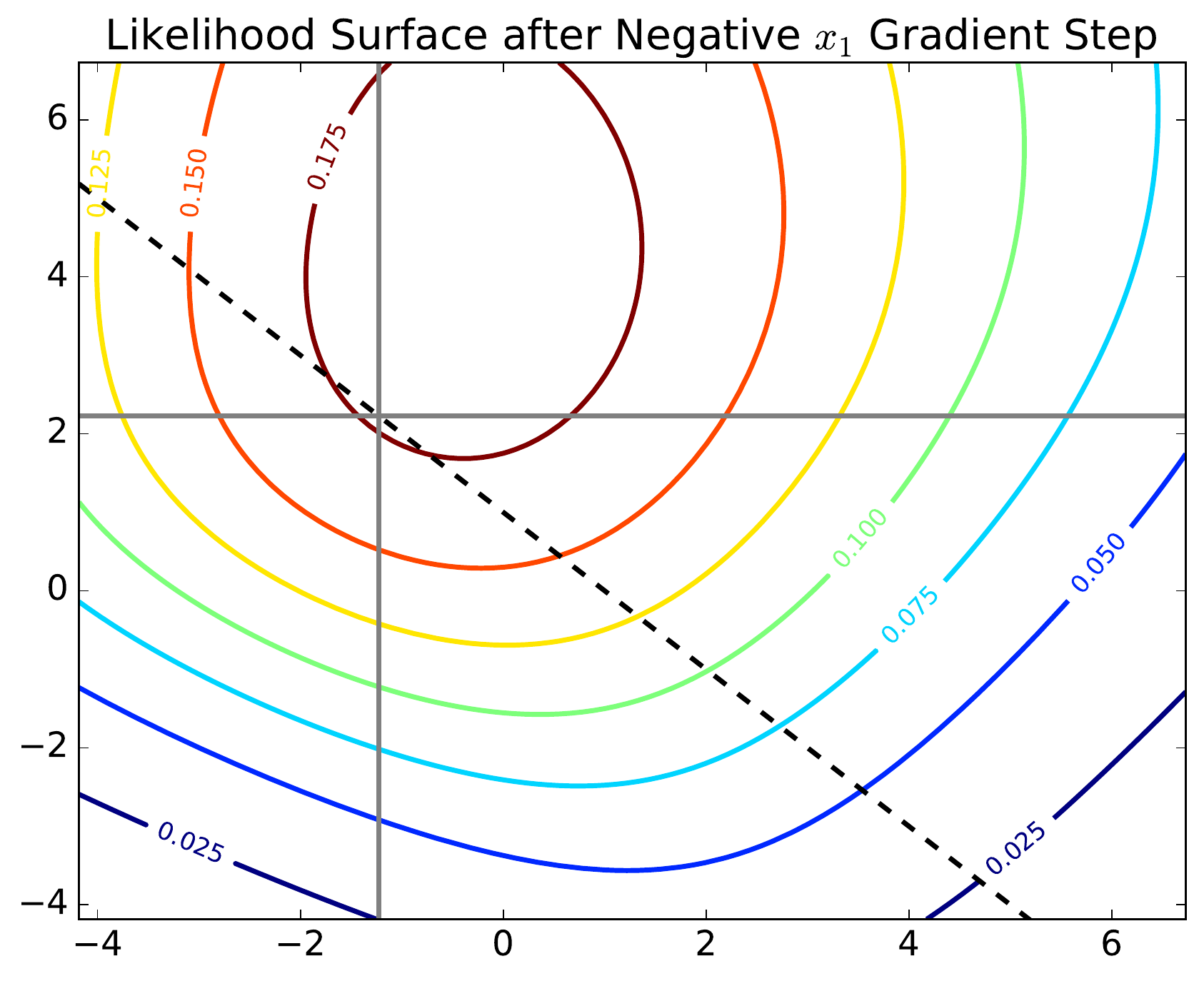}
    \\
    {\small (a) $\bx^\star = (0.151, 0.849)$} &
    {\small (b) $\bx^\star = (-1.23, 2,23)$}
  \end{tabular}
  \caption{\label{fig:eq_ml_example} Example maximum-likelihood
    surfaces $\ell_i(\bx)$ before and after taking a small step on all
    parameters in the negative gradient direction for $x^\star_1$ with
    constraint $\ones^T \bx = 1$.}
\end{figure}

Following \secref{sec:ineq_constraints} we can compute approximate
gradients as
\begin{align}
  \frac{\partial \bg_i}{\partial a_{jk}}
  &\approx \begin{cases}
    (H + \nabla^2 \phi(\bx^\star))^{-1} (\ell_i(\bx^\star) - 1) \be_k,
    & i = j
    \\
    \ell_j(\bx^\star) (H + \nabla^2 \phi(\bx^\star))^{-1} 
    \left(x^\star_k(\ba_j - \bar{\ba}) + \be_k\right),
    & i \neq j
  \end{cases}
\end{align}
and
\begin{align}
  \frac{\partial \bg_i}{\partial b_{j}}
  &\approx \begin{cases}
    0, & i = j
    \\
    \ell_j(\bx^\star) (H + \nabla^2 \phi(\bx^\star))^{-1}
    \left(\ba_j - \bar{\ba}\right), & i \neq j    
  \end{cases}
\end{align}
where $H$ and $\bar{\ba}$ are as defined in \secref{sec:ml_example},
and the second derivative for the log-barrier $\phi(\bx) = \log\left(1
- \|\bx\|^2\right)$ is given by
\begin{align}
  \nabla^2 \phi(\bx) &=
  \frac{4}{\left(1 - \|\bx\|^2\right)^2} \bx \bx^T
  - \frac{2}{1 - \|\bx\|^2} I_{n \times n}
\end{align}

Similar to previous examples, \figref{fig:ineq_ml_example} shows the
maximum-likelihood surfaces and corresponding solutions before and
after taking a gradient step (on all parameters) in the negative $x_1$
direction.

\begin{figure}
  \begin{tabular}{cc}
    \includegraphics[width=0.45\textwidth]{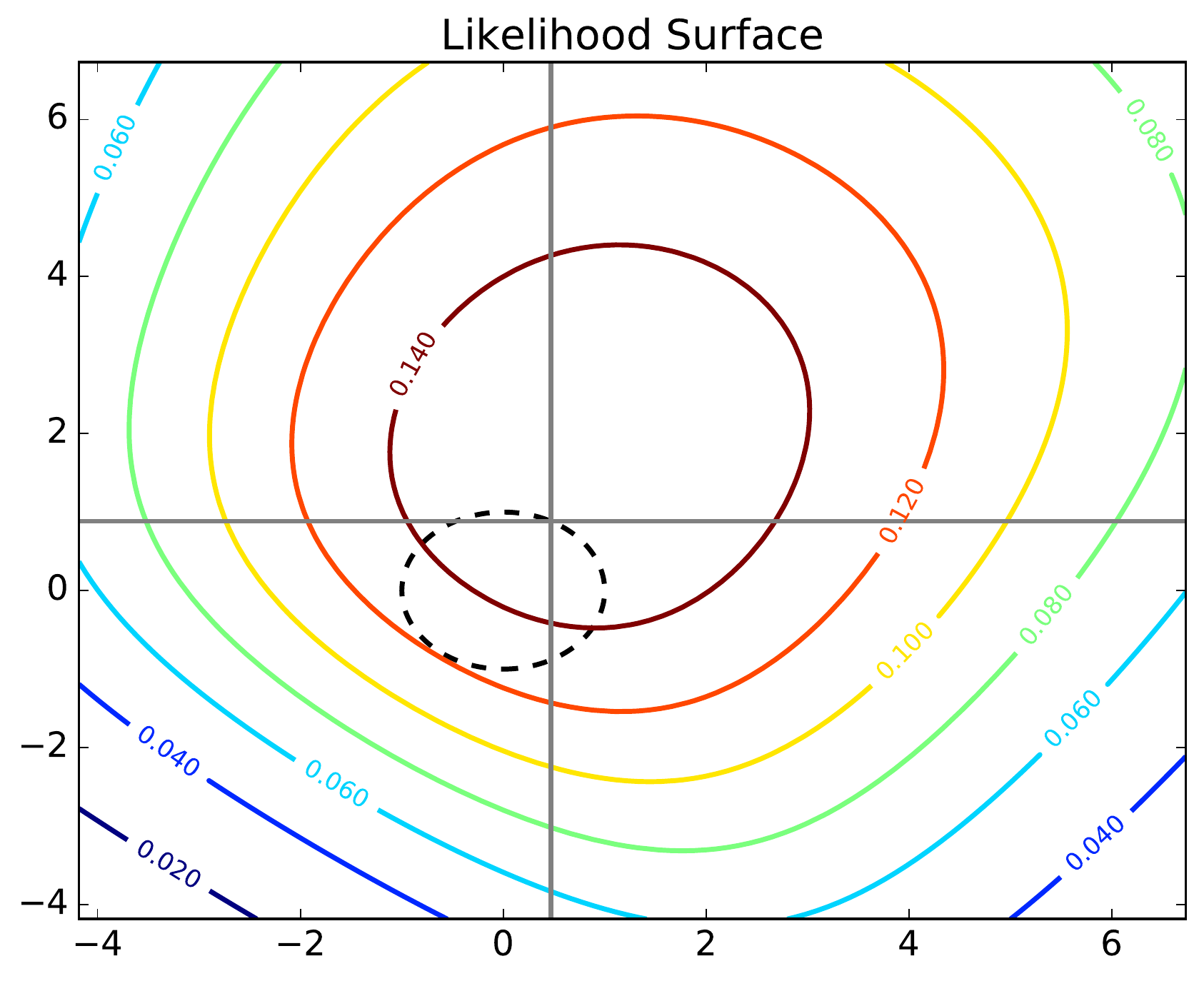} &
    \includegraphics[width=0.45\textwidth]{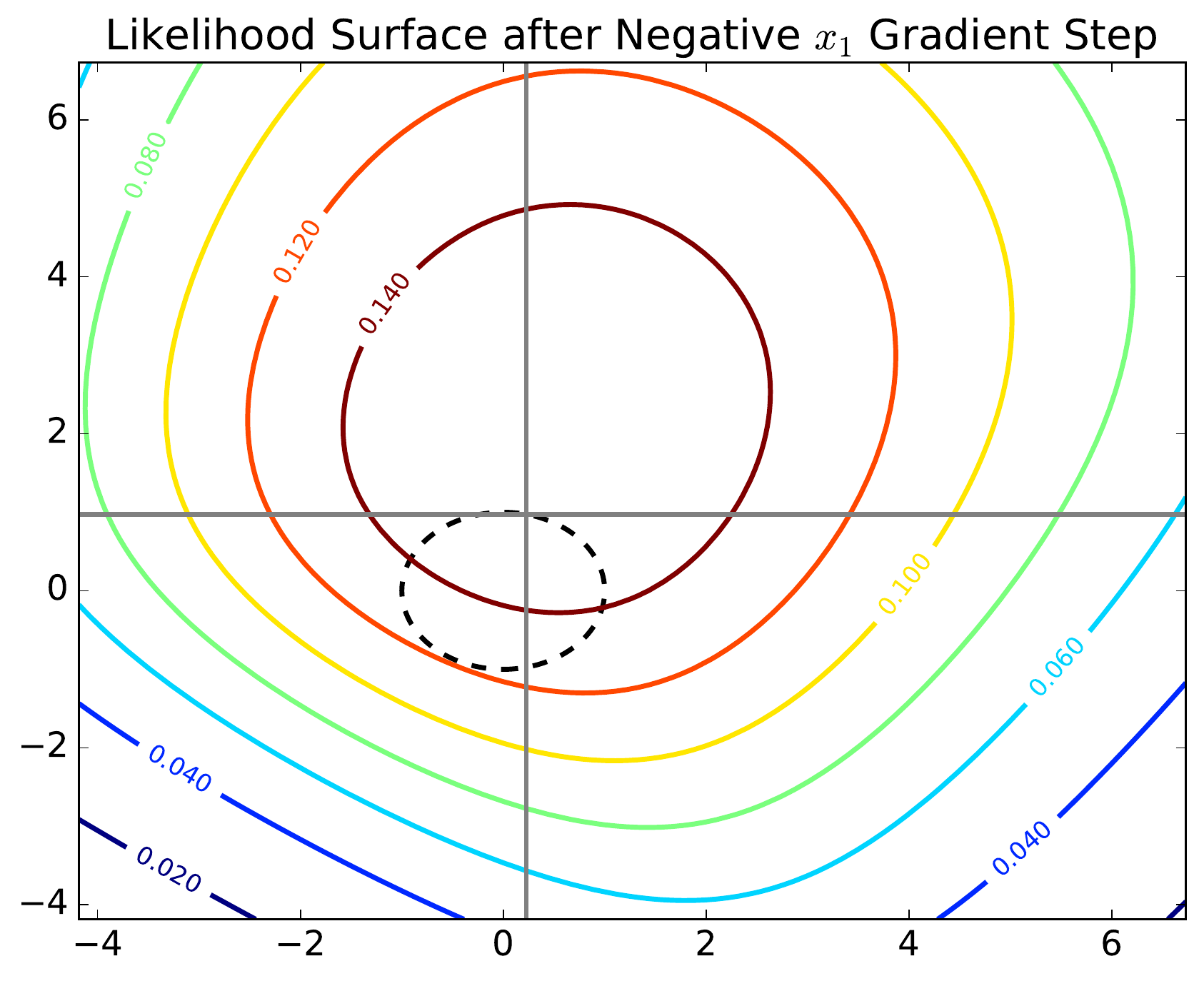}
    \\
    {\small (a) $\bx^\star = (0.467, 0.884)$} &
    {\small (b) $\bx^\star = (0.227, 0.973)$}
  \end{tabular}
  \caption{\label{fig:ineq_ml_example} Example maximum-likelihood
    surfaces $\ell_i(\bx)$ before and after taking a small step on all
    parameters in the negative gradient direction for $x^\star_1$ with
    constraint $\|\bx\|_2 \leq 1$.}
\end{figure}

\section{Bi-level Optimization Example}
\label{sec:bilevel_example}

We now build on previous examples to demonstrate the application of
the above results to the problem of bi-level optimization. We consider
the constrained maximum-likelihood problem from
\secref{sec:constrained_ml_example} with the goal of optimizing
parameter to achieve a desired location for the maximum-likelihood
feature vectors. We consider a three class problem over two
dimensional space. Here we use the constrained version of the problem
since for a three class soft-max classifier over two-dimensional space
there will always be a direction in which the likelihood tends to one
as the magnitude of the maximum-likelihood feature vector $\bx$ tends
to infinity.

Let the target location for the $i$-th maximum-likelihood feature
vector be denoted by $\bt_i$ and given. Optimizing the parameters of
the classifier to achieve these locations (with the maximum-likelihood
feature vectors constrained to the unit ball centered at the origin)
can be formalised as
\begin{align}
  \begin{array}{ll}
    \text{minimize}_{\Theta} & \displaystyle \frac{1}{2} \sum_{i=1}^{m} \| \bg_i(\Theta) - \bt_i \|^2
    \\
    \text{subject to} & \bg_i(\Theta) = \displaystyle \argmax_{\bx : \|\bx\|_2 \leq 1} \log \ell_i(\bx; \Theta)
  \end{array}
\end{align}

Solving by gradient descent gives updates
\begin{align}
  \theta^{(t + 1)} &\leftarrow \theta^{(t)} - 
  \eta \sum_{i=1}^{m} \left(\bg_i(\theta) - \bt_i\right)^T \bg'_i(\theta)
\end{align}
for any $\theta \in \Theta = \{(a_{ij}, b_i)\}_{i=1}^{m}$. Here $\eta$
is the step size.

Example likelihood surfaces for the three classes in our example are
shown in \figref{fig:bi_level} for both initial parameters and final
(optimized) parameters where we have set the target locations to be
evenly spaced around the unit circle. Notice that this is achieved by
the final parameter settings. Also shown in \figref{fig:bi_level}(c)
is the learning curve (in log-scale). Here we see a rapid decrease in
the objective in the first 20 iterations and final convergence (to
within $10^{-9}$ of the optimal value) in under 100 iterations.

\begin{figure}
  \centering
  \begin{tabular}{cc}
    \begin{minipage}{0.58\textwidth}
      \centering
      \includegraphics[width=\textwidth]{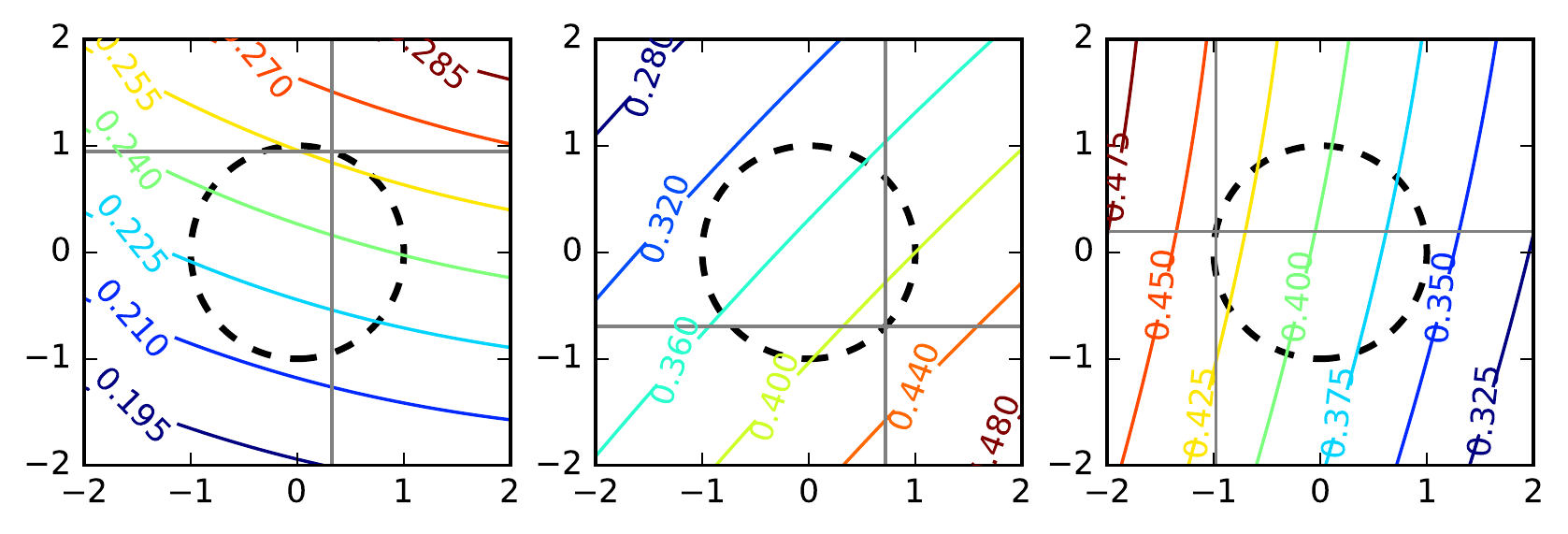} \\
      {\small (a) Initial} \\
      \includegraphics[width=\textwidth]{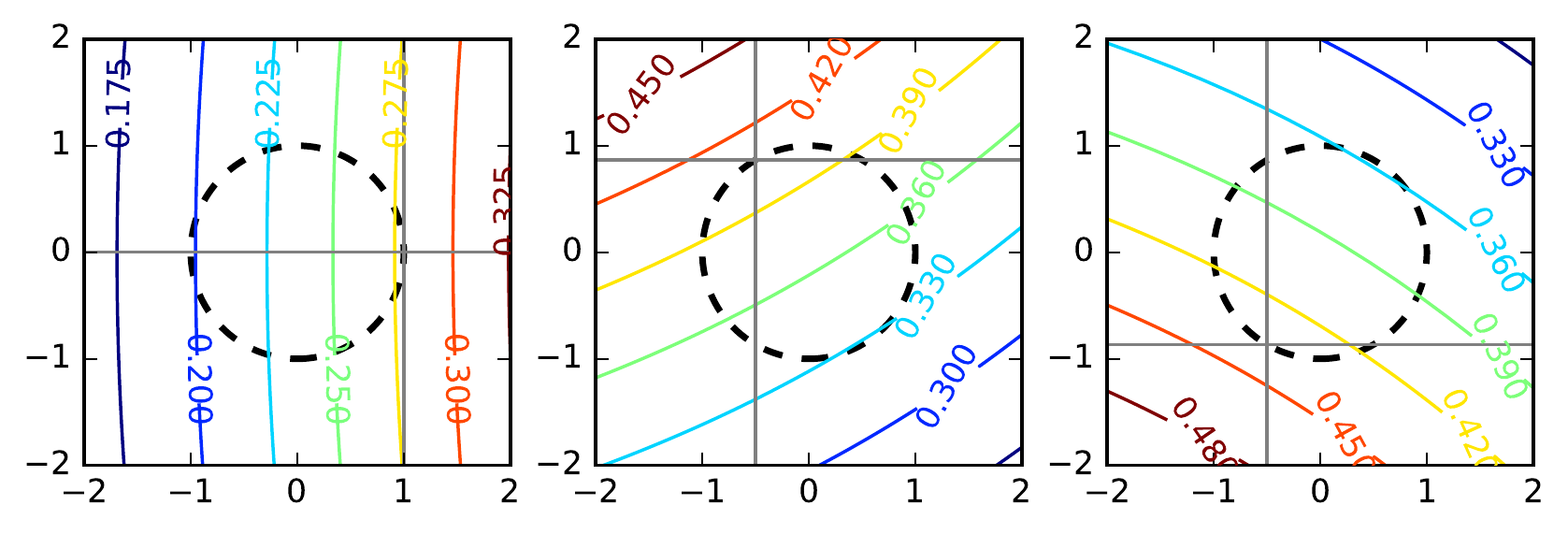} \\
      {\small (b) Final}
    \end{minipage}
    &
    \begin{minipage}{0.40\textwidth}
      \centering
      \includegraphics[width=\textwidth]{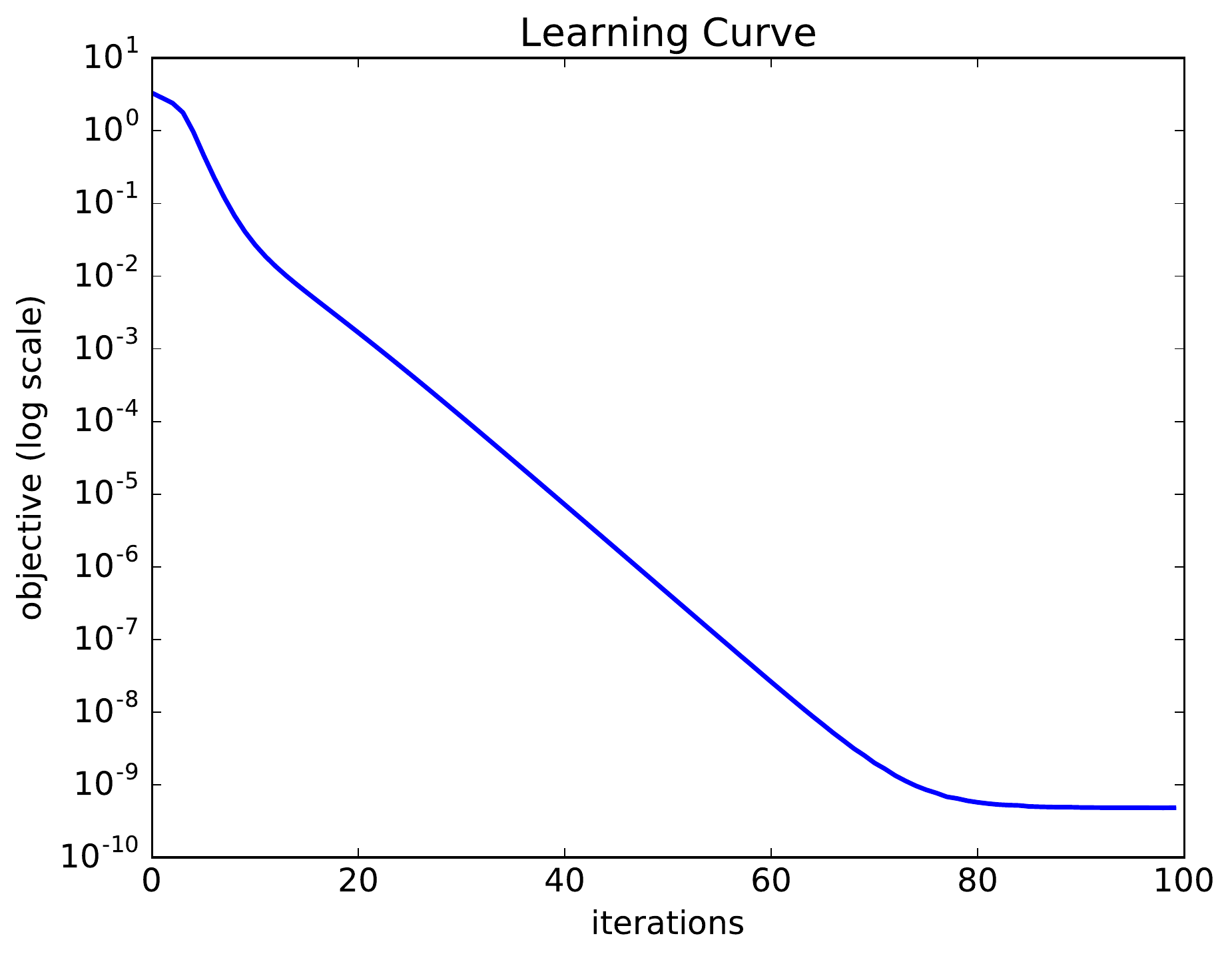} \\
      {\small (c) Learning Curve}
    \end{minipage}
  \end{tabular}
  \caption{\label{fig:bi_level} Example bi-level optimization to
    position maximum-likelihood points on inequality constrained
    soft-max model. Shown are the initial and final maximum-likelihood
    surfaces $\ell_i(\bx)$ for $i = 1, 2, 3$ and constraint $\|\bx\|_2
    \leq 1$ in (a) and (b), respectively. Also shown is the learning
    curve in (c) where the objective value is plotted on a log-scale
    versus iteration number.}
\end{figure}
  

\section{Discussion}

We have presented results for differentiating parameterized $\argmin$
and $\argmax$ optimization problems with respect to their
parameters. This is useful for solving bi-level optimization problems
by gradient descent~\cite{Bard:1998}. The results give exact gradients
but (i) require that function being optimized (within the $\argmin$ or
$\argmax$) be smooth and (ii) involve computing a Hessian matrix
inverse, which could be expensive for large-scale problems. However,
in practice the methods can be applied even on non-smooth functions by
approximating the function or perturbing the current solution to a
nearby differentiable point. Moreover, for large-scale problems the
Hessian matrix can be approximated by a diagonal matrix and still give
a descent direction as was recently shown in the context of
convolutional neural network (CNN) parameter learning for video
recognition via stochastic gradient descent~\cite{Fernando:ICML2016}.

The problem of solving non-smooth large-scale bi-level optimization
problems, such as in CNN parameter learning for video recognition,
present some interesting directions for future research. First, given
that the parameters are likely to be changing slowly for any
first-order gradient update it would be worth investigating whether
warm-start techniques would be effectivey for speeding up gradient
calculations. Second, since large-scale problems often employ
stochastic gradient procedures, it may only be necessary to find a
descent direction rather than the direction of steepest descent. Such
an approach may be more computationally efficient, however it is
currently unclear how such a direction could be found (without first
computing the true gradient). Last, the results reported herein are
based on the optimal solution to the lower-level problem. It would be
interesting to explore whether non-exact solutions could still lead to
descent directions, which would greatly improve efficiency for
large-scale problems, especially during the early iterations where the
parameters are likely to be far from their optimal values.

Models that can be trained end-to-end using gradient-based techniques
have rapidly become the leading choice for applications in computer
vision, natural language understanding, and other areas of artificial
intelligence. We hope that the collection of results and examples
included in this technical report will help to develop more expressive
models---specifically, ones that include optimization
sub-problems---that can still be trained in an end-to-end fashion. And
that these models will lead to even greater advances in AI
applications into the future.


{
  \setlength{\bibsep}{0pt}
  \bibliographystyle{abbrvnat}
  \bibliography{long,scene}
}

\end{document}